\documentclass[]{gsi_template}

\usepackage[toc,page,header]{appendix}
\usepackage{amsmath}
\usepackage{amssymb}
\usepackage{graphicx}
\usepackage{booktabs}
\usepackage{cleveref}

\AtBeginDocument{%
  }

\usepackage[utf8]{inputenc} 
\usepackage[T1]{fontenc}    
\usepackage{hyperref}       
\usepackage{url}            
\usepackage{booktabs}       
\usepackage{amsfonts}       
\usepackage{amsmath}       
\usepackage{nicefrac}       
\usepackage{microtype}      
\usepackage[table]{xcolor}         
\usepackage{xspace}        
\usepackage{balance}      
\usepackage{algorithm}
\usepackage{algpseudocode}
\usepackage{graphicx}
\usepackage{multirow}
\usepackage{makecell} 
\usepackage{array}
\usepackage{amsthm}
\usepackage{amssymb}
\newtheorem{proposition}{Proposition}
\newtheorem{remark}{Remark}

\definecolor{Highlight}{HTML}{CBC3E3}
\newcommand{\chl}{\cellcolor{Highlight}}

\newcommand{\name}{ProteinOPD\xspace} 

\newcommand{\best}[1]{\textbf{#1}}
\newcommand{\second}[1]{\underline{#1}}
\crefname{table}{Table}{Tables}

\title{ProteinOPD: Towards Effective and Efficient Preference Alignment for Protein Design}

\author[1\dag]{Yulin Zhang}
\author[2\dag]{He Cao}
\author[1]{Zihao Jiang}
\author[3]{Chenyi Zi}
\author[4]{Zhipeng Zhou}
\author[2]{Zijing Liu}
\author[2]{Yu Li}
\author[3]{Jia Li}
\author[1*]{Ziqi Gao}
\affiliation[1]{Tsinghua University}
\affiliation[2]{International Digital Economy Academy}
\affiliation[3]{Hong Kong University of Science and Technology (Guangzhou)}
\affiliation[4]{Nanyang Technological University}
\contribution[\dag]{Equal Contributions}
\contribution[*]{Corresponding Author}
\date{\today}
\checkdata[Code \& Data]{\url{https://github.com/THU-AI4S/ProteinOPD}}
\correspondence{Ziqi Gao \url{ziqigao@sz.tsinghua.edu.cn}}

\abstract{
Designing proteins with desired functions or properties represents a core goal in synthetic biology and drug discovery. 
Recent advances in protein language models (PLMs) have enabled the generation of highly designable protein sequences, while preference alignment provides a promising way to steer designs toward desired functions and properties.
Nevertheless, they often trigger catastrophic forgetting of pretrained knowledge, degrading basic designability and failing to balance multiple competing objectives. 
To address these issues, we draw inspiration from On-Policy Distillation (OPD), an advanced post-training method renowned for mitigating catastrophic forgetting through its mode-seeking nature.
In this work, we propose ProteinOPD, a multi-objective preference alignment framework that can effectively balance multiple preference objectives while maintaining the inherent designability of PLMs.
ProteinOPD adapts a pretrained PLM into preference-specific teachers and distills their knowledge into a shared student via token-level OPD on the student’s own trajectories. 
During this process, the student is aligned to a unique normalized geometric consensus of weighted teachers while ensuring bounded optimization under conflicts. This bridges the gap for OPD in multi-objective/teacher alignment.
Extensive experiments show that ProteinOPD achieves substantial gains on target preference objectives without compromising the designability, with an 8× training speedup over RL-based alignment competitors.
}

\begin{document}
\maketitle
\vspace{-20pt}

\section{Introduction}
Proteins are the central functional molecules of living systems and serve as key substrates for synthetic biology~\cite{khalil2010synthetic, yan2023applications}, therapeutic development~\cite{cubillos2021engineering, ebrahimi2023engineering}, and drug discovery~\cite{hie2024efficient, lu2024novo}. However, natural functional proteins occupy only a minuscule portion of the exponentially vast sequence space. While recent advances in protein language models (PLMs)~\cite{ferruz2022protgpt2, madani2020progen, lv2025prollama, nijkamp2023progen2} have enabled the generation of highly designable sequences, the \textit{de novo} design of proteins with tailored, user-specified properties remains a fundamental and persistent challenge in biotechnology.

In response, protein preference alignment (PPA) methods have emerged, aiming to align PLMs to generate sequences with desired properties or functions. Advanced alignment methods can be broadly categorized into two classes: \textbf{post-training} and \textbf{test-time steering}. Post-training methods reshape a PLM's output distribution by directly updating its parameters. The most straightforward approach is supervised fine-tuning (SFT)~\cite{luo2023flexible, munsamy2024conditional, madani2023large}, which adapts the PLM on a curated set of sequences pre-screened to satisfy the target objective. Despite its simplicity, the mode-covering objective of SFT  tends to constrain generation to the training distribution. These constraints prevent the model from simultaneously achieving high novelty and optimal target properties. Another prevalent class of post-training methods is reinforcement learning (RL)~\cite{hou2026property, wang2025proteinzero, liu2025controllable, stocco2024guiding, widatalla2024aligning}, which updates the PLM's policy using reward signals corresponding to the property scores of generated outputs. Although RL achieves effective alignment with desired objectives, its sequence-level supervision requires extensive rollouts, making training computationally expensive. Furthermore, when faced with new preferences, an RL policy can quickly become suboptimal, and drastic policy drift may occur, undermining the designability inherent in the PLM. Here, we use designability to refer to a model's ability to generate sequences that remain consistent with the natural protein distribution while retaining structural plausibility. The second major class, test-time steering~\cite{huang2025steering, adams2025mechanistic}, refers to methods that guide a PLM to generate high-quality sequences during inference without modifying its parameters. Although it offers remarkable computational efficiency, its reliance on linear combinations of steering vectors can be insufficient for navigating the complex, non‑linear trade‑offs when optimizing multiple competing objectives. Hence, we raise a central question—\textit{how can we efficiently align a PLM with multiple competing preferences without compromising its designability?}

\begin{figure}[!htp]
    \centering
    \includegraphics[width=1.0\linewidth]{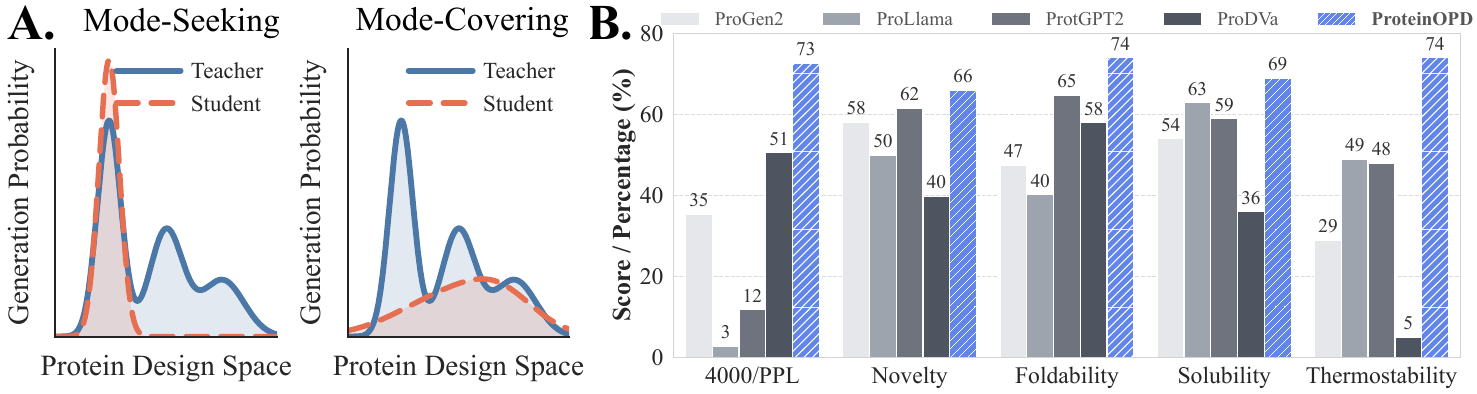}
    \caption{\textbf{A.} We show the superiority of the mode-seeking nature inherent to OPD. \textbf{B.} Performance comparison with advanced protein design and alignment methods. Our proposed ProteinOPD method achieves optimal performance in both designability and multi-property objectives.}
    \label{fig:fig1}
\end{figure}

On-policy Distillation (OPD) is valued in large language model (LLM) alignment due to its dual capability: \textbf{efficiently} adapting to new preferences while \textbf{effectively} resisting catastrophic forgetting~\cite{agarwal2024policy, zhao2026self, zeng2026glm}.
Specifically, OPD enables a student to learn from token-level supervision provided by a teacher model on the student's own generated trajectories. 
Unlike mode-covering methods such as offline distillation or SFT~\cite{zhang2026towards, tajwar2024preference}, OPD exhibits a mode-seeking nature~\cite{agarwal2024policy}. As shown in Figure \ref{fig:fig1}(a), OPD guides the student model to converge to the sharper and higher-reward modes of the teacher, rather than spreading probability mass across suboptimal options. This helps mitigate exposure bias and, in multi-preference scenarios, promotes clearer Pareto-optimal solutions. Moreover, compared to the sparse sequence-level rewards in conventional RL, OPD’s dense token-level supervision enables more efficient and stable optimization. Consequently, the advantages of OPD align well with the key challenges of PPA. Motivated by this, we explore OPD for protein design and leverage it as an algorithmic backbone for effective and efficient PPA.

In this work, we propose ProteinOPD, a multi-objective preference alignment framework that effectively balances multiple preference objectives while preserving the intrinsic designability of pretrained PLMs. ProteinOPD first performs the SFT, which adapts a pretrained PLM into different preference-specific teachers and then distills their preference signals into a shared student through on-policy distillation. To address the research gap in multi-objective (multi-teacher) OPD, we propose the generalized OPD that aligns the student with a geometric consensus distribution, formulated as a normalized product-of-experts over multiple teachers. 
Theoretically, the resulting consensus distribution yields bounded optimization signals even when teachers conflict, improving training stability under noisy and antagonistic preference objectives.
Empirically, ProteinOPD demonstrates the best alignment performance under both single- and multi-objective scenarios, while causing the minimal harm to designability. When using ProtGPT2 as the base model, ProteinOPD reduces PPL by 83.7\%, and enhances its foldability, solubility, and thermostability by 14.8\%, 16.9\%, and 54.2\%, respectively. ProteinOPD matches the alignment performance of RL at 1$/$8 of the training time.
\section{Related Work}
\paragraph{\textbf{Protein Sequence Design.}}
Existing protein sequence design methods can be broadly divided into unconditional and condition-controlled generation. Unconditional methods, such as ProtGPT2~\cite{ferruz2022protgpt2} and RITA~\cite{hesslow2022rita}, aim to generate protein sequences without explicit design constraints, thus exploring previously uncharted protein sequence space. In contrast, condition-controlled methods introduce extra signals to guide generation. ProGen~\cite{madani2020progen} uses biological tags that are defined in the pre-training stage for annotation-conditioned generation, while ProLLaMA~\cite{lv2025prollama} extends protein language modeling to both generation and understanding tasks through continued pre-training and instruction tuning, supporting both unconditional and family-conditional generation. More recent text-guided methods, such as Pinal~\cite{dai2024toward}, ProDVa~\cite{liu2025protein}, and PAAG~\cite{yuan2025annotation}, further use natural-language descriptions or functional annotations to guide protein design.
Collectively, these generative models have substantially advanced protein design by enabling the generation of novel and biologically plausible protein sequences. However, they are primarily trained to model natural protein distributions or to follow predefined conditioning signals, which do not necessarily align their outputs with the specific objectives required in protein engineering \cite{stocco2026steering}, highlighting the significance of protein preference alignment.

\paragraph{\textbf{Protein Preference Alignment.}}
Recent studies have investigated protein preference alignment for steering PLMs toward property- or function-aware sequence generation. Existing approaches mainly fall into two categories. Post-training methods modify model parameters after pretraining: SFT~\cite{luo2023flexible, munsamy2024conditional} trains PLMs on curated sequences satisfying target criteria, whereas RL~\cite{hou2026property, wang2025proteinzero, liu2025controllable} optimizes generation policies using property-based rewards. Test-time steering~\cite{huang2025steering, adams2025mechanistic} provides an alternative by guiding generation during inference while keeping model parameters fixed, typically through steering vectors constructed from positive and negative datasets associated with the target preference. Despite their effectiveness in certain settings, these methods still face several limitations. SFT may provide limited alignment when such data are scarce or biased~\cite{madani2023large}. RL allows for direct optimization of property-based rewards. However, its training process is often inefficient and may induce catastrophic forgetting, thereby weakening the pre-acquired designability and degrading the structural and evolutionary quality of the generated proteins ~\cite{stocco2024guiding, widatalla2024aligning}. Conversely, test-time steering avoids parameter updates, but its effectiveness depends heavily on the quality of steering directions and is difficult to align with multi-objective preferences. To address these limitations, we introduce a multi-objective on-policy distillation framework that aligns PLMs with multiple preferences in an effective and efficient manner while preserving their pretrained designability.
\section{Preliminaries}
\subsection{On-Policy Distillation}
\label{pre:opd}
On-Policy Distillation (OPD) computes supervision provided by teacher policy $p_T$ on trajectories sampled from the current student policy $p_S$.
Given a conditioning problem $x \sim \mathcal{D}_x$, the student samples a trajectory
\(y = (y_1, \dots, y_N) \sim p_{S}(\cdot \mid x)\), where \(N \triangleq |y|\) denotes the rollout length.
Both policies are then evaluated on the student-generated trajectory, yielding two next-token distributions at each step $n$: $p_S(\cdot \mid x,y_{<n})$ and $p_T(\cdot \mid x,y_{<n})$. 

A standard formulation minimizes the trajectory-averaged, token-wise divergence
\begin{equation}
\mathcal{L}_{\mathrm{OPD}}(\theta)
=
\mathbb{E}_{x \sim \mathcal{D}_x,\,
y \sim p_{S}(\cdot \mid x)}
\left[
\sum_{t=1}^{N}
D\!\left(
p_{S}(\cdot \mid x, y_{<t})
\,\middle\|\,
p_{T}(\cdot \mid x, y_{<t})
\right)
\right],
\label{eq:opd_token}
\end{equation}
where $D$ can be any distribution divergence measure. Unlike offline supervised fine-tuning (SFT), which frequently suffers from exposure bias and distribution shifts, this token-level decomposition ensures that the student receives corrective signals directly on its actively visited state space. This property establishes a robust foundation for transferring generative behaviors, which we subsequently extend to address multi-objective preference alignment.

\begin{figure}[!ht]
    \centering
    \includegraphics[width=1.0\linewidth]{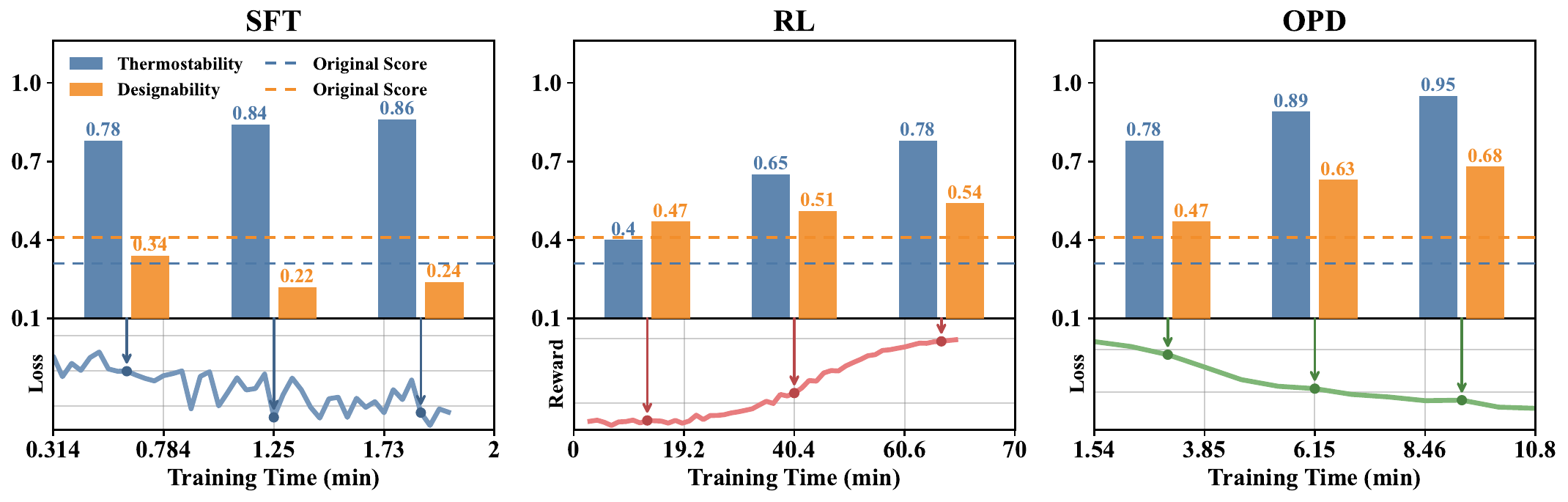}
    \caption{Preliminary comparative analysis of three popular alignment methods for improving thermostability. We track the loss/reward trend and evaluate three intermediate checkpoints on thermostability (alignment target) and designability. Dashed horizontal lines denote the capability of the base model (i.e., ProLLaMA) before alignment.}
    \label{fig:fig-2}
\end{figure}

\subsection{Preliminary Analysis}
Before introducing our method, we empirically analyze the efficiency and efficacy of OPD. For comparison, we investigate two other representative paradigms for protein preference alignment: SFT and RL.
For all three methods, following MoMPNN~\cite{hou2026property}, we employ TemBERTure~\cite{rodella2024temberture} as the oracle to assess the sequence's thermostability.
For SFT, we sample protein sequences from UniRef50 and use TemBERTure to estimate their thermostability. We select 100 high-quality data entries for SFT, and \textbf{the resulting model (after SFT) serves as the teacher in OPD}. For RL, we employ GRPO \cite{shao2024deepseekmath} with ProLLaMA~\cite{lv2025prollama} as the base model and use TemBERTure as the reward function. 
We take designability and thermostability as the primary metrics, and report the training time accordingly.

These exploratory experiments reveal complementary limitations of SFT and RL. 
As shown in Figure~\ref{fig:fig-2}, SFT achieves strong thermostability alignment at the cost of severe degradation in designability, suggesting catastrophic forgetting of pretrained protein knowledge. This also aligns with the mode‑covering nature of SFT. In contrast, RL preserves the pretrained model's designability, but the improvement of thermostability quickly saturates. 
For efficiency, RL is clearly the most time‑consuming, due to repeated oracle evaluations and its sequence‑level (sparse) supervision.
Since SFT requires minimal training time and aligns effectively with new preferences, the resulting SFT model naturally serves as the teacher in OPD. Although SFT has limitations in preserving designability, OPD’s resistance to forgetting ensures that the distillation retains both the student’s designability and the teacher’s preference scores.

Overall, our preliminary results show that OPD achieves stronger thermostability improvement effectively while mitigating catastrophic forgetting in a time-efficient way. This motivates us to further extend OPD from single-objective alignment to the multi-objective alignment setting.

\section{Method}
\paragraph{\textbf{Overview.}}
In this section, we introduce ProteinOPD, a method for multi-objective preference alignment in protein design, as illustrated in Figure~\ref{fig:pipeline}. The method first adapts a pre-trained protein language model into multiple preference-specific teachers. \textbf{For single-objective alignment}, we apply the standard on-policy distillation framework, distilling a teacher model into the student via token-level distributional alignment on the student's own rollout trajectories. \textbf{For multi-objective alignment}, the student is aligned with a geometric consensus distribution defined as a normalized product-of-experts (PoE) over all teachers. Moreover, we emphasize the general applicability of this multi-teacher OPD, which extends the standard OPD framework to the practical scenario of multi-teacher distillation. 

\paragraph{\textbf{Teacher Model Construction.}}For each target preference, we construct a specialized teacher through an efficient oracle-guided adaptation process. First, we sample sequences from UniRef50 and score them using a property-specific oracle. Then, the pre-trained PLM is adapted via supervised fine-tuning (SFT) on a small high-scoring subset $\mathcal{D}_p$ of only a few hundred sequences:
{
\begin{equation}
\phi^{*}
=
\arg\min_{\phi}
\sum_{(x,y)\in \mathcal{D}_p}
\mathcal{L}_{\mathrm{SFT}}\bigl(f(x;\phi),y\bigr),
\end{equation}
}

where the pretrained PLM is parameterized by $\phi$. In this work, we choose ProtGPT2~\cite{ferruz2022protgpt2} and ProLLaMA~\cite{lv2025prollama} as the pretrained PLMs for unconditional and conditional generation, respectively. This process requires minimal oracle annotations and short training time, avoiding the costly online queries in RL. The adapted PLM is then used as a preference-specific teacher, providing preference-biased token-level guidance for OPD.

\begin{figure}[!ht]
    \centering
    \includegraphics[width=1.0\linewidth]{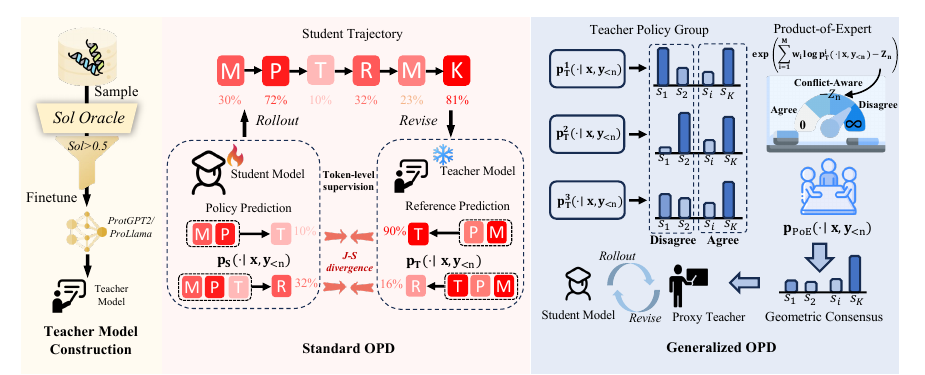}
    \caption{\textbf{Overview of our method.} Our framework comprises three stages: teacher model construction, single-objective alignment via standard OPD, and multi-objective alignment via generalized OPD. ProteinOPD flexibly adapts to each scenario, enjoying the effectiveness and efficiency of standard OPD for single-objective alignment and employing our proposed generalized OPD to manage competing objectives in multi-objective settings.}
    \label{fig:pipeline}
\end{figure}

\subsection{Single-Objective Alignment}

We instantiate the OPD formulation in Section~\ref{pre:opd} on autoregressive decoder-only PLMs, where
{
\begin{equation}
p_{\theta}(y \mid x)
=
\prod_{n=1}^{N}
p_{\theta}(y_n \mid x, y_{<n}).
\label{eq:ar_factorization}
\end{equation}
}

Here, \(x\) denotes the conditioning context for protein generation, such as a functional or family description. 
The single-preference training objective directly follows the token-level OPD objective in Eq.~\eqref{eq:opd_token}.
In all experiments, we instantiate the generic divergence \(D\) as the generalized Jensen--Shannon divergence:
{
\begin{equation}
\mathrm{JSD}_{\beta}(p_S \| p_T)
=
\beta \,
\mathrm{KL}\!\left(
p_S \,\middle\|\, m
\right)
+
(1-\beta)\,
\mathrm{KL}\!\left(
p_T \,\middle\|\, m
\right),
\label{eq:jsd_loss}
\end{equation}
}

where
$
m
=
\beta p_S + (1-\beta)p_T,
\label{eq:jsd_mixture}
$ is the interpolated mixture distribution, 
and we set \(\beta=0.5\) by default. This full-vocabulary logit distillation provides dense, token-level feedback, making preference alignment effective and efficient.

\subsection{Multi-Objective Alignment with Generalized OPD}
Real-world protein design frequently necessitates the optimization of \textbf{multiple competing objectives}, a challenge that standard OPD is ill-equipped to handle. To bridge this gap, we introduce a multi-teacher strategy into the OPD framework. We consider the multi-teacher OPD setting with a student policy $p_S$ and \(M\) teacher policies $\{p_T^i\}_{i=1}^M$. Let $\{w_i\}_{i=1}^M$ be non-negative teacher weights satisfying $\sum_{i=1}^M w_i = 1$. Our goal is to construct a unified token-level target that captures the shared preference structure across teachers, and then align the student to this target under the standard OPD framework.

\paragraph{\textbf{Geometric Consensus Induced by the Multi-teacher Objective.}}

A central question in multi-teacher distillation is whether the student should imitate \emph{coverage} or \emph{consensus}. 
Arithmetic mixtures emphasize coverage, assigning high probability to tokens preferred by any teacher. 
Here, instead, we seek a target that retains only the shared support across teachers. 
For each student-visited state \((x,y_{<n})\), we construct a local consensus distribution by finding the optimal distribution $q$:
{
\begin{equation}
\arg\min_{q}
\sum_{i=1}^{M}
w_i\,
\mathrm{KL}\!\left(
q
\,\middle\|\,
p_T^i(\cdot \mid x,y_{<n})
\right).
\label{eq:local_consensus}
\end{equation}
}

\begin{proposition}
The solution to Eq.~\eqref{eq:local_consensus} is a normalized product-of-experts in Eq.~\ref{eq:normalized_poe}. 
\end{proposition}
The proof is straightforward and provided in Appendix~\ref{app:proofs}.


\paragraph{\textbf{Normalized PoE target.}}
As mentioned above, the solution to Eq.~\eqref{eq:local_consensus} is the normalized product-of-experts (PoE), which is a closed-form weighted geometric mean of the teacher distributions. This consensus distribution thus naturally upweights tokens that receive joint support across teachers:
\begin{equation}
p_{\mathrm{PoE}}(\cdot \mid x,y_{<n})
=
\mathrm{softmax}
\left(
\sum_{i=1}^{M}
w_i
\log p_T^i(\cdot \mid x,y_{<n})
\right)
=
\exp\!\left(
\sum_{i=1}^M w_i \log p_T^i(\cdot \mid x,y_{<n}) - Z_n
\right),
\label{eq:normalized_poe}
\end{equation}
where
$
Z_n
=
\log
\sum_{v \in \mathcal{V}}
\exp\!\left(
\sum_{i=1}^M w_i \log p_T^i(v \mid x,y_{<n})
\right).
\label{eq:zt_def}
$
Thus, computing the consensus target only requires averaging teacher log-probabilities with weights $\{w_i\}_{i=1}^M$, followed by a single \texttt{logsumexp} over the vocabulary.

\paragraph{\textbf{Disagreement from Normalization.}}
The normalization term \(Z_n\) also captures how much probability mass is jointly shared across teachers. Applying the weighted Arithmetic Mean-Geometric Mean (AM--GM) inequality, we have the following property.

\begin{proposition}
  For any prefix $(x,y_{<n})$,
  $Z_n
  =
  \log
  \sum_{v \in \mathcal{V}}
  \prod_{i=1}^M
  p_T^i(v \mid x,y_{<n})^{w_i}
  \le 0$,
  with equality if and only if all teachers with non-zero weight induce the same next-token distribution at prefix \(x,y_{<n}\).
\end{proposition}
\begin{remark}
We interpret $-Z_{n}$ as a natural, token-level measure of teacher disagreement (or attribute conflict). By definition, when all teachers agree on an amino acid substitution, $-Z_{n}$ approaches 0. Conversely, when teachers severely conflict, $-Z_{n}$ increases. This provides a cost-free indicator to monitor and quantify the internal tug-of-war among different protein properties during training, without requiring any additional computational overhead.
\end{remark}

\paragraph{\textbf{Training Objective.}}
The multi-preference alignment follows the same OPD protocol as the single-preference setting, but replaces the single-teacher target with the normalized PoE target in Eq.~\eqref{eq:normalized_poe}. 
{
\setlength{\abovedisplayskip}{2pt}
\setlength{\belowdisplayskip}{2pt}
\begin{equation}
\mathcal{L}^{\mathrm{multi}}(\theta)
=
\mathbb{E}_{x \sim \mathcal{D}_x,\, y \sim p_S(\cdot \mid x)}
\left[
\sum_{n=1}^{N}
\mathrm{JSD}_{\beta}\left(
p_{\theta}
\,\middle\|\,
p_{\mathrm{PoE}}
\right)
\right].
\label{eq:final_mt_objective}
\end{equation}
}
The normalized PoE defines a unified teacher policy, combining all preference-specific experts. 
\section{Experiments}
In this section, we show that \name consistently achieves the best alignment-designability trade-off for multi-objective preference alignment in the main result. Further analysis is organized around three research questions. We examine two representative settings: unconditional protein design with ProtGPT2, which explores uncharted protein design space, and conditional protein design with ProLLaMA, which evaluates whether generated sequences satisfy specified design conditions.
\subsection{Experimental Setups}
\paragraph{\textbf{Data Construction.}}
We construct data for both unconditional and conditional protein generation from UniRef50~\cite{suzek2015uniref}. 
For the conditional setting, we follow ASPO~\cite{huang2025steering} and restrict the data to the lysozyme-like superfamily, instantiating a representative family-conditioned generation task. 
Each sequence is annotated with preference scores from property-specific oracles. 
For each preference objective, we retain high-scoring sequences using predefined thresholds to construct preference-specific adaptation data. 
Details of the data construction are provided in Appendix~\ref{app:data_construct}.
\paragraph{\textbf{Evaluation Metrics.}}
\label{exp:metrics}
We evaluate generated proteins from three perspectives: designability, preference alignment, and overall multi-objective performance. 
For designability, we report perplexity (PPL), Novelty, and, in the conditional setting, ProTrek score~\cite{su2024protrek}. PPL measures the plausibility of generated sequences under a protein language model, while Novelty quantifies their dissimilarity to reference protein databases. We consider two variants: Novelty-U, computed against UniProtKB~\cite{uniprot2023uniprot} and shared across all baselines, and Novelty-T, computed against the training set of each baseline. In conditional generation, ProTrek further evaluates the semantic consistency between generated proteins and the input conditions, following ProDVa~\cite{liu2025protein}.  
For preference alignment, we use property-specific oracles to score generated proteins. Foldability is measured by ESMFold~\cite{lin2022language}-derived pLDDT and PAE~\cite{jumper2021highly}, solubility (Sol.) by Protein-Sol~\cite{hebditch2017protein}, and thermostability (Thermo.) by TemBERTure~\cite{rodella2024temberture}. 
Overall performance is measured by hypervolume (HV), summarizing trade-offs across designability and preference-alignment metrics. Details are provided in Appendix~\ref{app:metric_details}.

\begin{table}[!ht]
    \centering
    \caption{Performance comparison of baselines and \name in {multi-objective} preference alignment in the \textit{unconditional} setting. 
    The best and second-best scores are marked in \textbf{bold} and \second{underline}. }
    \resizebox{\textwidth}{!}{
    \begin{tabular}{lcccccccccc}
        \toprule
        \multirow{3}{*}{\large\textbf{Models}} 
        & \multicolumn{3}{c}{\textbf{Designability}} 
        & \multicolumn{6}{c}{\textbf{Preference Alignment Ability}}
        & \multicolumn{1}{c}{\textbf{Overall}}\\
        \cmidrule[0.5pt](lr){2-4} 
        \cmidrule[0.5pt](lr){5-10}
        \cmidrule[0.5pt](lr){11-11}
        & \makecell{PPL \\ (\(\downarrow\))} 
        & \makecell{Novelty-U \\ (\(\uparrow\))} 
        & \makecell{Novelty-T \\ (\(\uparrow\))} 
        & \makecell{pLDDT \\ (\(\uparrow\))} 
        & \makecell{\%>70 \\ (\(\uparrow\))} 
        & \makecell{pAE \\ (\(\downarrow\))} 
        & \makecell{\%<10 \\ (\(\uparrow\))} 
        & \makecell{Sol \\ (\(\uparrow\))} 
        & \makecell{Thermo \\ (\(\uparrow\))} 
        & \makecell{HV \\ (\(\uparrow\))} \\
        \cmidrule[0.5pt](lr){1-1}
        \cmidrule[0.5pt](lr){2-2}
        \cmidrule[0.5pt](lr){3-4}
        \cmidrule[0.5pt](lr){5-8}
        \cmidrule[0.5pt](lr){9-9}
        \cmidrule[0.5pt](lr){10-10}
        \cmidrule[0.5pt](lr){11-11}
         \texttt{Random (U)} & 1765.28 & 57.13 & -- & 24.02 & 0.00 & 24.15 & 0.00 & 0.34 & 0.02 & 0.05 \\ 
        \texttt{Random (E)} & 1747.35 & 59.84 & -- & 26.9 & 0.39 & 23.99 & 0.78 & 0.44 & 0.03 & 0.01 \\
        \midrule
        ESM2 & 1467.04 & 52.75 & 60.72 & 47.52& 5.68 & 25.7 & 0.55 & 0.40 & 0.42 & 0.15 \\
        ESM3 & 353.71 & 51.21 & \second{72.66} & 61.30 & 27.49 & 19.68 & 20.52 & 0.31 & 0.56 & 0.21 \\
        ProGen2 & 113.04 & \best{85.76} & 58.70 & 47.44 & 10.55 & 24.7 & 6.25 & 0.54 & 0.29 & 0.44 \\
        ProLLaMA & 1445.8 & 55.07 & 42.95 & 40.17 & 0.0 & 19.4 & 0.39 & 0.63 & 0.49 & 0.06 \\
        ProtGPT2 & 339.57 & \second{59.97} & 69.17 & 64.77 & 42.58 & 17.69 & 24.61 & 0.59 & 0.48 & 0.33 \\ 
        \midrule
        Pinal & 542.49 & 46.44 & 45.68 & 64.32 & \second{52.79} & \second{15.92} & \second{40.28} & 0.41 & 0.28 & 0.37 \\
        ProteinDT & 1255.16 & 50.69 & 52.25 & 42.75 & 11.37 & 24.77 & 5.82 & 0.39 & 0.36 & 0.29 \\
        ProDVa & \second{79.42} & 2.21 & 37.58 & 58.04 & 9.38 & 18.88 & 7.81 & 0.36 & 0.05 & 0.02 \\
        \midrule 
        ASPO & 804.45 & 51.57 & 56.77 & 47.99 & 1.95 & 19.63 & 3.12 & \best{0.69} & \second{0.68} & 0.17 \\ 
        MoMPNN & 748.73 & 42.50 & 65.47 & \second{70.19} & 50.66 & 16.78 & 35.05 & 0.57 & 0.65 & \second{0.46} \\ 
        \chl \textbf{ProteinOPD} & \chl \best{55.25}  & \chl 54.96  & \chl \best{75.68} & \chl \best{74.37} & \chl \best{61.72} & \chl \best{13.91} & \chl \best{45.7} & \chl \best{0.69} & \chl \best{0.74} & \best{0.62}\chl  \\ 
        \bottomrule
    \end{tabular}
    }
     \label{tab: unconditional-multi}
\end{table}

\paragraph{\textbf{Baselines.}}
General protein design baselines include pretrained PLMs, including ESM2~\cite{lin2022language}, ESM3~\cite{hayes2025simulating}, ProGen2~\cite{nijkamp2023progen2}, ProLLaMA~\cite{lv2025prollama}, and ProtGPT2~\cite{ferruz2022protgpt2}, and text-conditioned models, including Pinal~\cite{dai2024toward}, ProteinDT~\cite{liu2025text}, and ProDVa~\cite{liu2025protein}. We also include two random baselines: Random(U), which samples amino acids uniformly, and Random(E), which samples from the empirical amino acid distribution in SwissProt following ProDVa~\cite{liu2025protein}. 
Preference-alignment baselines cover two paradigms: post-training and test-time steering. For post-training, we use SFT on ProtGPT2~\cite{ferruz2022protgpt2} and ProLLaMA~\cite{lv2025prollama}. These models also serve as preference-specific teachers in \name. For test-time steering, we compare with ASPO~\cite{huang2025steering}, which steers PLM generation with property-specific vectors.  We further include MoMPNN~\cite{hou2026property} as a cross-paradigm reference baseline, since few multi-objective alignment baselines are available for sequence-only protein generation. 
For conditional protein design, all baselines are evaluated under the same lysozyme-like family condition. 
Details of the baseline settings are provided in Appendix~\ref{app:baseline_details}.

\subsection{Performance on Multi-Objective Preference Alignment}
\label{exp:main-results}

The main results of the paper are provided by Table~\ref{tab: unconditional-multi}. We highlight our key findings as follows: 

\textbf{(1) \name achieves the best overall performance, with an HV of 0.62. }
Compared with MoMPNN, the strongest cross-paradigm reference in our evaluation, ProteinOPD improves HV by 34.8\%. 
It also outperforms the best general design baseline, ProGen2, by 40.9\%, and the strongest text-conditioned baseline, Pinal, by 67.6\%, indicating a stronger global trade-off between designability and preference alignment. 

\textbf{(2) \name consistently improves all target preferences while preserving designability. }
Compared to the base model ProtGPT2, \name improves the pLDDT, Sol, and Thermo by 14.8\%, 16.9\%, and 54.2\%, respectively. 
Meanwhile, it maintains sequence-plausibility and novelty, reducing PPL by 83.7\% and improving Novelty-T by 9.4\%. 
Compared with the second-best baseline, \name achieves the best pLDDT of 74.37, improving over MoMPNN by 6.0\%, matches the best Sol of 0.69, and achieves the best Thermo of 0.74, improving over ASPO by 8.8\%. 

\textbf{(3) Existing baselines are limited by their alignment mechanisms. }
General design models provide strong priors but lack explicit preference optimization. 
Text-conditioned models offer condition-level guidance, but weakly supervise specific preferences. Pinal remains lower than \name by 0.28 in solubility and 0.46 in thermostability.  
ASPO uses fixed test-time steering, which is bounded by the pretrained generator and cannot update model parameters. 
MoMPNN optimizes multiple objectives under an inverse-folding setting. The backbone constraints improve structure-related metrics, but their designability remains weaker than \name, with 693.48 higher PPL and 12.46 lower Novelty-T.  
Overall, \name yields a more balanced trade-off across designability and preference alignment.

Performance comparison of the baselines and ProteinOPD in multi-objective preference alignment in the \textbf{conditional setting} is provided in Appendix~\ref{app: multi-prefer-condition}. 

\subsection{Analysis}
\label{exp:RQ}
Beyond the main result, we perform a three-level analysis of \name: (\textbf{RQ1}) to examine the trade-off behind its HV gains and whether it expands the Pareto frontier; (\textbf{RQ2}) to investigate, in the single-objective case, if OPD transfers preferences while mitigating the designability loss from SFT teachers; and (\textbf{RQ3}) to assess its efficiency for training and data needs for teacher construction.

\begin{figure}[!ht]
    \centering
    \vspace{-0.2cm}
    \includegraphics[width=1.0\linewidth]{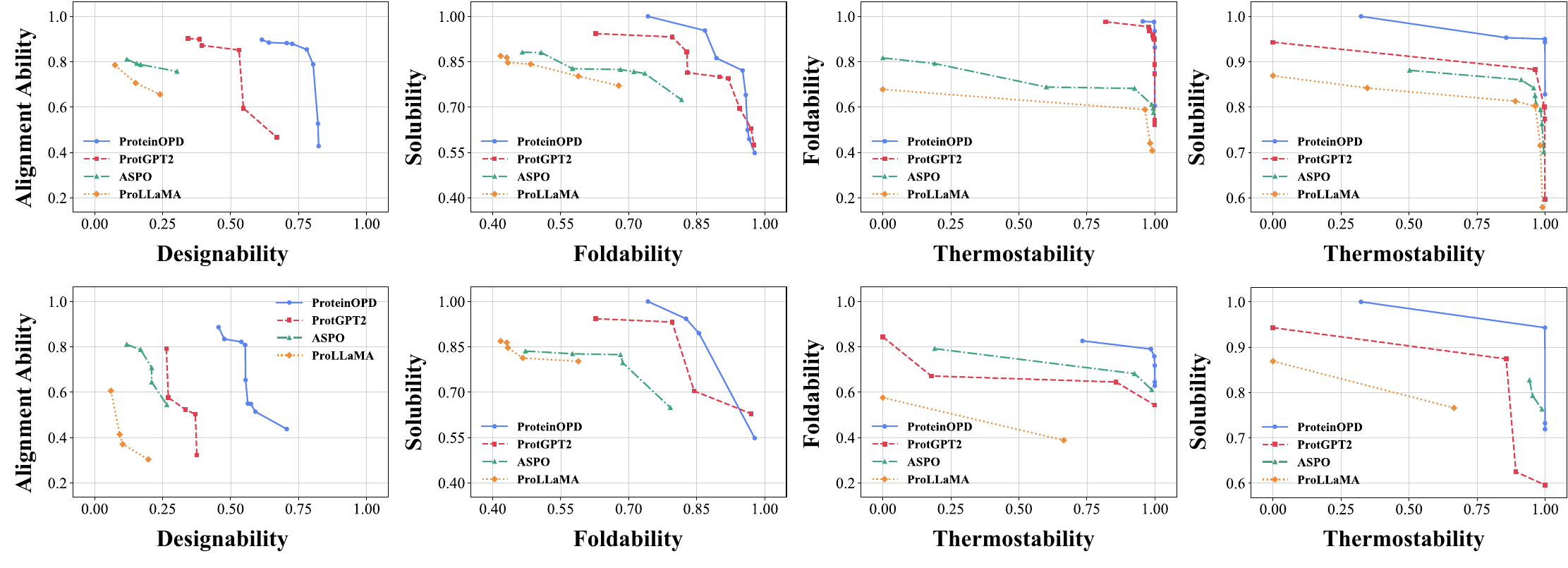}
    \vspace{-0.5cm}
    \caption{\textbf{Pareto-front analysis}. The \textbf{top row} reports fronts over all generated sequences, and the \textbf{bottom row} reports fronts after filtering sequences with Novelty-U \(>0.7\). 
Designability is computed from normalized PPL, and alignment ability is the average of pLDDT, solubility, and thermostability.}
    \label{fig:pareto-front}
\end{figure}

\paragraph{\textbf{Pareto Frontier Expansion in Multi-Objective Preference Alignment (RQ1).}}As illustrated in Figure~\ref{fig:pareto-front}, across both designability--alignment and pairwise preference comparisons, \name consistently shifts the frontier outward, indicating stronger attainable trade-offs among objectives. 
We further evaluate high-novelty sequences with Novelty-U \(>0.7\). 
\name retains more favorable fronts, showing its advantage extends to exploratory regions of the sequence space rather than remaining confined to sequences close to UniProtKB. This behavior is essential for unconditional protein design, where the goal is to explore uncharted protein design space.

\begin{figure}[h!]
    \centering
    \includegraphics[width=1.0\linewidth]{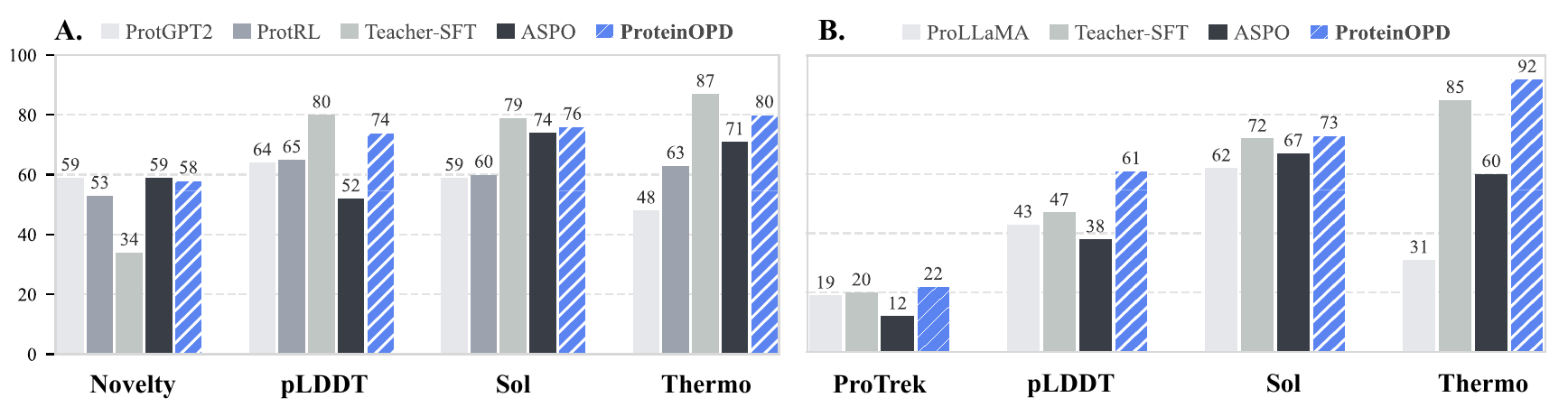}
    \vspace{-0.6cm}
    \caption{
Single-objective alignment results in the unconditional setting (\textbf{A}) and the conditional setting (\textbf{B}). Teacher-SFT refers to the constructed teacher model adapted to the high-scoring subset. 
}\label{fig:single-objective}
\end{figure}

\paragraph{\textbf{Performance on Single-Objective Preference Alignment (RQ2).}}Here, we analyze the results of \name in the single-objective setting. Figure~\ref{fig:single-objective}A shows that OPD can transfer preference signals from SFT teachers while substantially mitigating the designability loss caused by direct fine-tuning. 
In the unconditional setting, Teacher-SFT improves pLDDT, Sol, and Thermo substantially, but at the cost of a 42.4\% novelty drop. In contrast, \name preserves most of these property gains while reducing novelty by only 1.7\% relative to ProtGPT2. This gap supports our claim that OPD mitigates the forgetting behavior of SFT-based alignment. 

In the conditional setting, the advantage of \name is even more pronounced because alignment must be achieved without breaking condition consistency, as shown in Figure~\ref{fig:single-objective}B. Compared with Teacher-SFT, ProteinOPD improves ProTrek Score by 10.0\%, indicating better consistency with the original conditional generation capability. 
It also further improves pLDDT, solubility, and thermostability by 29.8\%, 1.4\%, and 8.2\%, respectively. 
The behavior is consistent with the nature of OPD: the student is corrected on its own visited states, rather than being forced to cover the full distribution of the high-scoring SFT data. 
Compared with ProtRL and ASPO, ProteinOPD provides more stable single-objective alignment while better preserving either novelty in unconditional generation or condition fidelity in conditional generation.

\paragraph{\textbf{Efficiency in Computation and Data (RQ3).}}

\begin{figure}[!ht]
    \centering
    \includegraphics[width=1.0\linewidth]{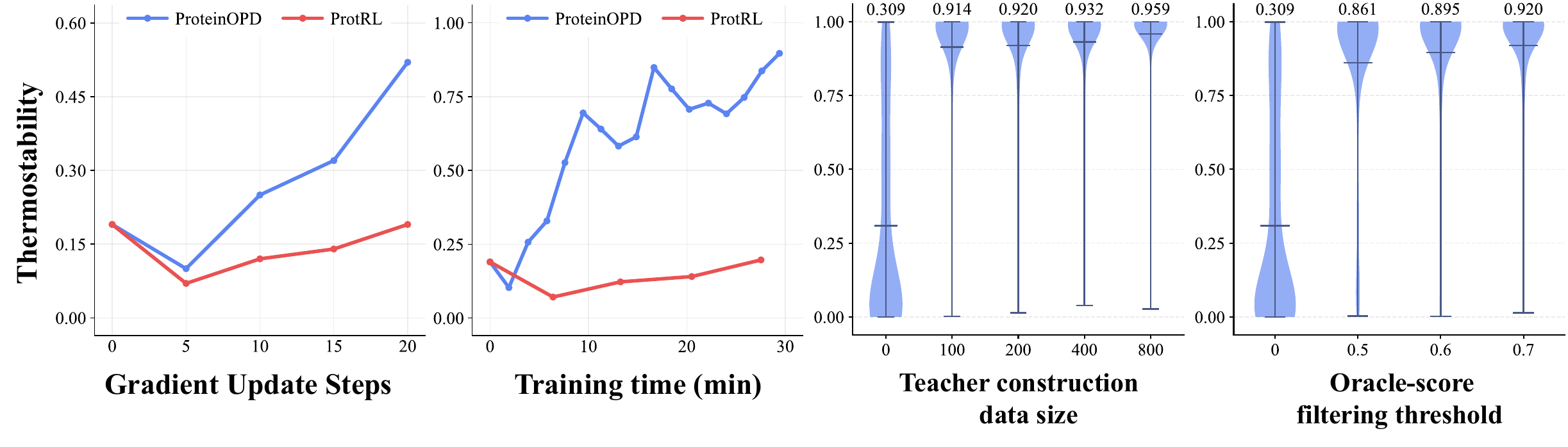}
    \vspace{-0.7cm}
    \caption{\textbf{Efficiency analysis.} Online cost is evaluated by comparing \name and ProtRL over gradient steps and wall-clock time. Offline cost is evaluated by varying the teacher-construction data size and oracle-score filtering threshold. The last two panels report the best evaluation performance within 100 training steps, and top numbers denote mean predicted thermostability.}
    \label{fig:efficiency}
\end{figure}

Beyond alignment quality, a practical protein design method must be computationally viable, especially given expensive oracle queries. We therefore assess whether \name maintains the efficiency advantages in data and training computation. As shown in Figure~\ref{fig:efficiency}, \name proves efficient in both online and offline alignment.
For online cost, \name improves thermostability much faster than ProtRL. 
Within 20 gradient steps, \name increases the thermostability score from about 0.19 to 0.52, whereas ProtRL remains around 0.19. 
Under the same wall-clock comparison, \name reaches about 0.70 within 10 minutes and further improves to 0.89 by 30 minutes, while ProtRL stays below 0.20. \name reaches a thermostability score of 0.25 in 3.86 minutes, whereas ProtRL requires 31.2 minutes to reach a comparable score. 
This corresponds to an 8$\times$ training speed-up.
This indicates that token-level OPD provides a denser and more efficient learning signal than reward-based online optimization. 
For offline cost, we find that using 100 filtered sequences already yields a mean thermostability of 0.914, close to that of a larger quantity. 
As for data quality, when raising the filtering threshold from 50 to 70, the mean score increases from 0.861 to 0.920. 
These results show that \name enables effective teacher construction with limited offline data, with additional gains from oracle-based filtering.

\begin{figure}[!ht]
    \centering
    \includegraphics[width=.8\linewidth]{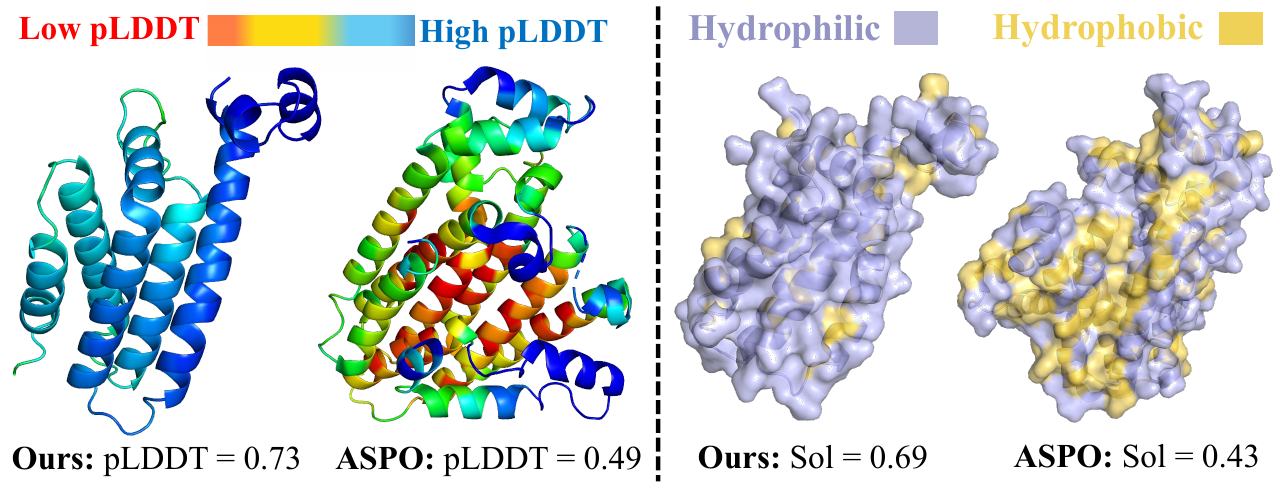}
    \caption{Residue-level visualization.}
    \vspace{-0.3cm}
    \label{fig:vis}
\end{figure}
\paragraph{\textbf{Case Visualization.}}We present generated examples of ProteinOPD and ASPO, ensuring that the thermostability of both exceeds 0.95, and show their scores of pLDDT and solubility. Moreover, we ensure novelty against UniRef by requiring a max seq. identity below 5\%. Thus, ProteinOPD enables effective multi-attribute alignment and the creation of novel sequences outside current databases in practical use.


\section{Conclusion}
We present \name, a multi-objective preference alignment framework for protein language models that aligns protein generation with desired preferences while preserving pretrained designability. 
\name constructs preference-specific teachers through lightweight oracle-guided adaptation and transfers their preference guidance to a shared student via token-level on-policy distillation on student-generated trajectories. 
For multi-objective preference alignment, \name extends OPD to the multi-teacher setting with a normalized product-of-experts consensus target over preference-specific teachers. Experiments across unconditional and conditional protein generation show that \name achieves superior designability--preference trade-offs compared with existing baselines, while delivering an 8$\times$ training speedup over RL-based alignment. These results establish OPD as an effective and efficient mechanism for preference-aligning protein language models.

\clearpage
\bibliographystyle{plainnat}
\bibliography{main}

\clearpage
\beginappendix
\appendix
\section{Details of Metrics}
\label{app:metric_details}

This section provides detailed definitions of the evaluation metrics used in our experiments. 

\paragraph{\textbf{Perplexity.}}
Perplexity (PPL) measures how likely a generated protein sequence is under a reference protein language model. 
Given a generated sequence $x=(x_1,x_2,\ldots,x_L)$, we compute its perplexity as
\begin{equation}
\mathrm{PPL}(x)
=
\exp\left(
-\frac{1}{L}
\sum_{t=1}^{L}
\log p_{\mathrm{ref}}(x_t \mid x_{<t})
\right),
\end{equation}
where $p_{\mathrm{ref}}$ denotes the next-token distribution predicted by the reference PLM. 
In our experiments, we use ProtGPT2 as the evaluator. 
A lower PPL indicates that the generated sequence is more consistent with the distribution learned by the reference PLM, and is therefore used as a proxy for sequence plausibility.

\paragraph{\textbf{Novelty.}}
Novelty measures how different the generated proteins are from known or training sequences. 
Let $\mathcal{R}$ denote the reference set, such as the training set or a database of natural proteins, and let $\mathrm{sim}(P,R)$ denote the sequence similarity between a generated sequence $P$ and a reference sequence $R \in \mathcal{R}$. 
The novelty of $P$ is defined as its dissimilarity to the nearest reference sequence:
\begin{equation}
\mathrm{Novelty}(P)
=
1-
\max_{R \in \mathcal{R}}
\mathrm{sim}(P,R),
\end{equation}
where $\mathrm{sim}(P,R)$ denotes the normalized sequence similarity between $P$ and $R$, computed using MMseqs2. 
The final novelty score is obtained by averaging over all generated proteins:
\begin{equation}
\mathrm{Novelty}(\mathcal{G})
=
\frac{1}{|\mathcal{G}|}
\sum_{P \in \mathcal{G}}
\mathrm{Novelty}(P).
\end{equation}
A higher novelty score indicates that the generated proteins are less similar to known sequences and are less likely to be memorized from the training data.

\paragraph{\textbf{ProTrek Score.}}
For conditional protein generation, we use the ProTrek Score to evaluate whether the generated protein is semantically aligned with the input condition. 
Given a textual condition $c$ and a generated protein sequence $P$, ProTrek encodes them into a shared representation space using its text encoder and protein sequence encoder, respectively. 
The ProTrek Score is computed as the cosine similarity between the two embeddings:
\begin{equation}
\mathrm{ProTrek}(c,P)
=
\cos\left(
\tau_c(c), \tau_p(P)
\right),
\end{equation}
where $\tau_c(c)$ denotes the embedding of the textual condition and $\tau_p(P)$ denotes the embedding of the generated protein sequence. 
A higher ProTrek Score indicates stronger semantic consistency between the generated sequence and the specified condition.

\paragraph{\textbf{pLDDT.}}
The predicted Local Distance Difference Test (pLDDT) is used to evaluate the foldability of generated proteins. 
For a generated sequence $P$ of length $L$, a structure prediction model produces a residue-level confidence score
$\mathbf{s}=(s_1,s_2,\ldots,s_L)$, where each $s_i \in [0,100]$ measures the confidence of the predicted local structure around residue $i$. 
We report the average pLDDT over all residues:
\begin{equation}
\mathrm{pLDDT}(P)
=
\frac{1}{L}
\sum_{i=1}^{L}
s_i.
\end{equation}
Higher pLDDT values indicate more confident structure predictions and are commonly used as a proxy for protein foldability.

\paragraph{\textbf{PAE.}}
Predicted Aligned Error (PAE) measures the expected positional error between residue pairs in a predicted structure. 
For a generated protein of length $L$, the structure prediction model outputs a matrix
$\mathbf{E}\in\mathbb{R}^{L\times L}$, where $E_{ij}$ represents the predicted error of residue $i$ when the predicted and actual structures are aligned on residue $j$. 
We compute the average PAE as
\begin{equation}
\mathrm{PAE}(P)
=
\frac{1}{L^2}
\sum_{i=1}^{L}
\sum_{j=1}^{L}
E_{ij}.
\end{equation}
Lower PAE values indicate more reliable global structure prediction, especially with respect to residue-pair geometry, domain packing, and long-range structural arrangement.
\paragraph{\textbf{Solubility.}}
We download the Protein-Sol oracle as the solubility predictor from \url{https://protein-sol.manchester.ac.uk/software}. The predictor outputs a score ranging from \(0\) to \(1\) for each sequence. 
\paragraph{\textbf{Thermostability.}}
We download the TemBERTure model from \url{https://github.com/ibmm-unibe-ch/TemBERTure} and use the \textbf{temBERTure\_CLS} mode. The model outputs a Thermophilic score ranging from \(0\) to \(1\) for each sequence. 
\paragraph{\textbf{Hypervolume.}}
Hypervolume (HV) measures the dominated volume of the Pareto set in the normalized objective space. 
For each generated protein $P$, let $m_k(P)$ denote its raw value on the $k$-th metric. 
We first apply a metric-specific transformation
\begin{equation}
\psi_k(m)=
\begin{cases}
\log(1+m), & \text{if } m \text{ is PPL},\\
m, & \text{otherwise},
\end{cases}
\end{equation}
where the $\log(1+\cdot)$ transformation compresses the large numerical range of PPL. 
We then orient all metrics such that larger values are better:
\begin{equation}
a_k(P)=s_k \psi_k(m_k(P)), 
\quad
s_k=
\begin{cases}
+1, & \text{for metrics to maximize},\\
-1, & \text{for metrics to minimize}.
\end{cases}
\end{equation}
Here, metrics to maximize include Novelty, pLDDT, solubility, thermostability, and ProTrek Score, while metrics to minimize is PPL.

The oriented scores are min-max normalized over all generated sequences from all compared methods under the same experimental setting:
\begin{equation}
z_k(P)=
\frac{a_k(P)-a_k^{\min}}
{a_k^{\max}-a_k^{\min}},
\quad
a_k^{\min}=\min_{P\in\mathcal{U}} a_k(P),
\quad
a_k^{\max}=\max_{P\in\mathcal{U}} a_k(P),
\end{equation}
where $\mathcal{U}$ denotes the union of generated sequences from all methods. 
The normalized objective vector is therefore
$\mathbf{z}(P)=(z_1(P),\ldots,z_K(P))\in[0,1]^K$.

Given the generated set $\mathcal{G}$ of a method, we extract the non-dominated subset
$\mathrm{ND}(\mathcal{G})$ in the normalized objective space. 
The HV is defined as
\begin{equation}
\mathrm{HV}(\mathcal{G})
=
\lambda_K
\left(
\bigcup_{P\in \mathrm{ND}(\mathcal{G})}
[\mathbf{0},\mathbf{z}(P)]
\right),
\end{equation}
where $\lambda_K(\cdot)$ denotes the $K$-dimensional Lebesgue measure and $\mathbf{0}$ is used as the reference point. 
A larger HV indicates that the generated sequences achieve a better trade-off across designability and preference-alignment objectives.
\section{Details of Experimental Setups}
\subsection{Data Construction Details}
\label{app:data_construct}
We construct preference-specific datasets using the same oracle models described in Section~\ref{exp:metrics}. 
For each sampled sequence, foldability is evaluated using ESMFold-derived pLDDT and pAE, thermostability using TemBERTure, and solubility using Protein-Sol. 

We apply separate filtering thresholds in the unconditional and conditional settings. 
For unconditional generation, sequences are sampled from UniRef50 and filtered independently for pLDDT, Sol and Thermo using 80, 0.7, and 0.5, respectively. 
For conditional generation, sequences are sampled from the conditional protein subset and independently filtered for pLDDT, Sol, and Thermo using thresholds of 80, 0.7, and 0.7, respectively. 
For all experiments, we use 200 sequences to finetune the pre-trained model unless otherwise specified. 
\subsection{Training Details}
\begin{table}[t]
\centering
\small
\caption{Training configuration for teacher construction and ProteinOPD.}
\label{tab:training_config}
\renewcommand{\arraystretch}{1.12}
\resizebox{\linewidth}{!}{
\begin{tabular}{lcccc}
\toprule
\textbf{Parameter} 
& \textbf{ProtGPT2 Teacher} 
& \textbf{ProLLaMA Teacher} 
& \textbf{ProteinOPD (ProtGPT2)} 
& \textbf{ProteinOPD (ProLLaMA)} \\
\midrule
Base Model 
& ProtGPT2 
& ProLLaMA 
& ProtGPT2 
& ProLLaMA \\

Training Stage 
& Teacher SFT 
& Teacher SFT 
& OPD 
& OPD \\

Fine-tuning Method 
& Prefix Tuning 
& LoRA Tuning 
& LoRA Tuning
& LoRA Tuning \\

\midrule
Learning Rate 
& $1\times10^{-3}$ 
& $2\times10^{-5}$ 
& $2\times10^{-5}$ 
& $2\times10^{-5}$  \\

Training Steps 
& 50
& 20
& 200 
& 100 \\

\midrule
Prefix Length 
& 100 
& -- 
& -- 
& -- \\

LoRA Rank ($r$) 
& -- 
& 8 
& 64 
& 8 \\

LoRA Alpha ($\alpha$) 
& -- 
& 16 
& 128 
& 16 \\

LoRA Target Modules 
& -- 
& \makecell{q\_proj, k\_proj, v\_proj, o\_proj,\\
gate\_proj, up\_proj, down\_proj}
& \makecell{c\_attn c\_proj c\_fc}
& \makecell{q\_proj, k\_proj, v\_proj, o\_proj,\\
gate\_proj, up\_proj, down\_proj} \\
\bottomrule
\end{tabular}
}
\end{table}
All training runs are conducted on 2 NVIDIA A800-SXM4-80GB GPUs using AdamW with bfloat16 precision and a warm-up ratio of 0.1. 
The maximum completion length is set to 512 tokens, with top-$k$ sampling of 200 and top-$p$ sampling of 0.95. 
For ProteinOPD, the student rollout temperature is set to 1.0, while teacher distributions are computed with a temperature of 0.7. 
The JSD loss uses $\beta=0.5$ for all OPD experiments.
For multi-teacher OPD, we use the same distillation weights in both unconditional and conditional settings. 
Specifically, the Foldability, Thermostability, and Solubility teachers are assigned weights of 0.3, 0.4, and 0.3, respectively. More implementation details are provided in Table~\ref{tab:training_config}. 
\subsection{Baseline Details
}
\label{app:baseline_details}
The multi-objective preference alignment baselines are limited. For ASPO, we adapt its test-time steering mechanism to the multi-objective setting by linearly combining property-specific steering vectors with predefined weights. 
Specifically, the steering vector is constructed as a weighted sum of the thermostability, solubility, and foldability vectors, with weights $1.0$, $0.5$, and $0.5$, respectively. 
This combined vector is then used to guide sequence generation toward multiple target preferences.

MoMPNN is included as a cross-paradigm reference baseline rather than a task-identical sequence-generation baseline. 
Unlike ProteinOPD, which generates protein sequences without requiring structural inputs, MoMPNN operates in an inverse-folding setting and requires backbone structures as design conditions. 
Since few multi-objective preference-alignment baselines are available for sequence-only protein generation, we include MoMPNN to provide an additional reference point for multi-objective optimization. 
We randomly sample 256 protein backbones and use MoMPNN to design one sequence for each backbone. 
The resulting sequences are then evaluated with the same oracle-based preference metrics and designability metrics as all other methods. 
Therefore, the comparison should be interpreted as a downstream performance reference across design paradigms rather than a strictly task-identical comparison.

Since ESM2 and ESM3 are auto-encoding protein language models rather than autoregressive generators, they cannot directly generate protein sequences from left to right. Following ASPO~\cite{huang2025steering}, we use an iterative masked-token infilling protocol to obtain generated sequences from these models. Specifically, we first sample reference protein sequences from UniRef50 as reference sequences. In each iteration, we randomly select and mask 10\% of the tokens without replacement, then generate them using the model. 

ProLLaMA is applicable to both unconditional and conditional generation. 
Therefore, we include ProLLaMA as a baseline in both settings. 
In the unconditional setting, it is used to generate protein sequences without additional family constraints. 
In the conditional setting, it is prompted with the lysozyme-like superfamily description. 

For text-guided baselines, we use fixed, property-explicit prompts to ensure that the target preference is directly provided to the model. 
For single-objective generation, the prompt template is 
``Design novel protein sequences with high [PROPERTY].''
where [PROPERTY] is instantiated as foldability, solubility, or thermostability. 
For multi-objective generation, we use 
``Design novel protein sequences with high foldability, high solubility, and high thermostability.'' 
In the conditional setting, we additionally include the family constraint:
``Design novel protein sequences from the lysozyme-like superfamily with high [PROPERTY].''
or its multi-objective counterpart. 

To ensure a fair comparison, all baselines are evaluated using the same generation and evaluation protocol unless otherwise specified. 
We generate 256 protein sequences for each method with temperature $0.7$, and a maximum generation length of 512 tokens. 
We evaluate generated sequences using designability and preference-alignment metrics. 
Specifically, pLDDT, pAE, ProTrek Score, Novelty, and PPL are computed based on PDFBench~\cite{kuang2025pdfbench}. 
For each metric, we report the mean value over all 256 generated sequences.
\section{Additional Experimental Results}
\subsection{Performance on Multi-objective Preference Alignment in the conditional setting}
\label{app: multi-prefer-condition}
\begin{table}[!ht]
    \caption{Performance comparison of baselines and \name in {multi-objective} preference alignment in the \textit{conditional} setting. The best and second-best scores are marked in \textbf{bold} and \second{underline}.}
    \label{tab:exp-1}
    \centering
    \resizebox{\textwidth}{!}{
    \begin{tabular}{lccccccc}
        \toprule
        \multirow{3}{*}{Models} & \multicolumn{2}{c}{Designability} & \multicolumn{2}{c}{Foldability} & \multicolumn{1}{c}{Solubility} & \multicolumn{1}{c}{Thermostability} & \multicolumn{1}{c}{HV} \\
        \cmidrule[0.5pt](lr){2-3} \cmidrule[0.5pt](lr){4-5} \cmidrule[0.5pt](lr){6-6} \cmidrule[0.5pt](lr){7-7} \cmidrule[0.5pt](lr){8-8}
        & \makecell{PPL \\ (\(\downarrow\))} &
        \makecell{ProTrek \% \\ (\(\uparrow\))} & \makecell{pLDDT \\ (\(\uparrow\))} & \makecell{pAE \\ (\(\downarrow\))} & \makecell{Sol \\(\(\uparrow\))} & \makecell{Thermo \\(\(\uparrow\))} & \makecell{HV \\ (\(\uparrow\))} \\
        \midrule
        ProLLaMA & 1473.68 & \second{19.18} & 43.03 & 20.78 & 0.62 & 0.31 & 0.09 \\
        \midrule
        ASPO & 1264.59 & 8.46 & 43.27 & 23.07 & \best{0.76} & \second{0.62} & \second{0.44} \\
        ProDVa & \second{953.33} & 18.95 & \second{56.08} & \second{20.41} & 0.30 & 0.13 & 0.05 \\
        \name & \chl \best{766.6}  & \chl \best{22.05}  & \chl \best{61.36} & \chl \best{15.56}  & \chl \second{0.64} & \chl \best{0.8} & \chl \best{0.61} \\ 
        \bottomrule
    \end{tabular}
    }
\end{table}
In the conditional multi-objective setting, \name achieves the strongest overall trade-off among competing objectives. 
Compared with the base ProLLaMA model, \name reduces PPL by 707.08 points, from 1473.68 to 766.60, and improves the ProTrek score by 2.87 points, from 19.18 to 22.05. 
This suggests that \name improves sequence plausibility while maintaining, and even strengthening, alignment with the conditional input. 
For foldability, \name improves pLDDT by 18.33 points and reduces pAE by 5.22 points compared with ProLLaMA, indicating more plausible predicted structures. 
Although ASPO achieves a higher solubility score than \name, \name improves thermostability by 0.18 over ASPO and achieves a 0.17 higher hypervolume score. 
These results indicate that \name does not merely optimize a single property, but provides a better overall balance across designability, foldability, solubility, and thermostability.

\subsection{Independent PPL Evaluation}
\label{app:independent_ppl}
Since ProtGPT2 is used as the primary evaluator for PPL in our main experiments, we further compute PPL using ProGen2 as an independent autoregressive protein language model. 
This evaluation provides an additional check on whether the improved designability is specific to the ProtGPT2 evaluator. 
ProteinOPD obtains a ProGen2-based PPL of 5.83, lower than that of the original ProtGPT2 generator (7.25). 
This result suggests that the generated sequences remain more plausible under an evaluator that is not used in either teacher construction or OPD training. 
\section{Proofs for Generalized OPD}
\label{app:proofs}

This section provides the proofs for the normalized product-of-experts target and the normalization-based disagreement measure used in generalized OPD. 
Let $\mathcal{V}$ denote the token vocabulary, and let $\Delta(\mathcal{V})$ denote the probability simplex over $\mathcal{V}$. 
For each teacher $i\in\{1,\ldots,M\}$, $p_T^i(\cdot\mid x,y_{<n})$ denotes its next-token distribution at this prefix. 
The teacher weights satisfy $w_i\ge 0$ and $\sum_{i=1}^M w_i=1$. 
Since PLMs parameterize next-token distributions through softmax outputs, we assume $p_T^i(v\mid x,y_{<n})>0$ for all $v\in\mathcal{V}$.

\paragraph{\textbf{Proof of Proposition 1.}}
At the prefix $(x,y_{<n})$, generalized OPD constructs the local consensus target by solving
\begin{equation}
q_n^\star
=
\arg\min_{q_n\in\Delta(\mathcal{V})}
\sum_{i=1}^M
w_i
\mathrm{KL}
\left(
q_n(\cdot)
\,\middle\|\,
p_T^i(\cdot\mid x,y_{<n})
\right).
\end{equation}
Expanding the objective gives
\begin{align}
\sum_{i=1}^M w_i
\mathrm{KL}
\left(
q_n
\,\middle\|\,
p_T^i(\cdot\mid x,y_{<n})
\right)
&=
\sum_{v\in\mathcal{V}}
q_n(v)
\left[
\log q_n(v)
-
\sum_{i=1}^M
w_i
\log p_T^i(v\mid x,y_{<n})
\right].
\end{align}
Define the normalized product-of-experts distribution
\begin{equation}
p_{\mathrm{PoE}}(v\mid x,y_{<n})
=
\exp
\left(
\sum_{i=1}^M
w_i
\log p_T^i(v\mid x,y_{<n})
-
Z_n
\right),
\end{equation}
where
\begin{equation}
Z_n
=
\log
\sum_{u\in\mathcal{V}}
\exp
\left(
\sum_{i=1}^M
w_i
\log p_T^i(u\mid x,y_{<n})
\right).
\end{equation}
By substituting this definition into the expanded objective, we obtain
\begin{align}
\sum_{i=1}^M w_i
\mathrm{KL}
\left(
q_n
\,\middle\|\,
p_T^i(\cdot\mid x,y_{<n})
\right)
&=
\sum_{v\in\mathcal{V}}
q_n(v)
\log
\frac{
q_n(v)
}{
p_{\mathrm{PoE}}(v\mid x,y_{<n})
}
-
Z_n \\
&=
\mathrm{KL}
\left(
q_n
\,\middle\|\,
p_{\mathrm{PoE}}(\cdot\mid x,y_{<n})
\right)
-
Z_n .
\end{align}
Since $Z_n$ is independent of $q_n$ and the KL divergence is non-negative, the minimum is achieved if and only if
\begin{equation}
q_n^\star(\cdot)
=
p_{\mathrm{PoE}}(\cdot\mid x,y_{<n}).
\end{equation}
Therefore, the solution to the weighted multi-teacher objective is exactly the normalized product-of-experts target in Eq.~(6). 
This proves Proposition~1.

\paragraph{\textbf{Proof of Proposition 2.}}
For the same prefix $(x,y_{<n})$, the normalization term in Eq.~(6) is
\begin{equation}
Z_n
=
\log
\sum_{v\in\mathcal{V}}
\prod_{i=1}^M
p_T^i(v\mid x,y_{<n})^{w_i}.
\end{equation}
For each token $v$, the weighted AM--GM inequality gives
\begin{equation}
\prod_{i=1}^M
p_T^i(v\mid x,y_{<n})^{w_i}
\le
\sum_{i=1}^M
w_i
p_T^i(v\mid x,y_{<n}).
\end{equation}
Summing both sides over the vocabulary yields
\begin{align}
\sum_{v\in\mathcal{V}}
\prod_{i=1}^M
p_T^i(v\mid x,y_{<n})^{w_i}
&\le
\sum_{v\in\mathcal{V}}
\sum_{i=1}^M
w_i
p_T^i(v\mid x,y_{<n}) \\
&=
\sum_{i=1}^M
w_i
\sum_{v\in\mathcal{V}}
p_T^i(v\mid x,y_{<n})
=
1 .
\end{align}
Taking the logarithm gives
\begin{equation}
Z_n \le 0.
\end{equation}
The equality condition of the weighted AM--GM inequality holds if and only if, for every token $v$ and every teacher with $w_i>0$, the values $p_T^i(v\mid x,y_{<n})$ are identical across teachers. 
Thus, $Z_n=0$ if and only if all non-zero-weight teachers induce the same next-token distribution at the prefix $(x,y_{<n})$; otherwise, $Z_n<0$. 
Equivalently, $-Z_n\ge 0$ is zero under complete agreement and increases as the shared probability mass among teachers decreases. 
This proves Proposition~2.
\section{Limitations}
In this work, we evaluate generated proteins using widely adopted computational Oracles for foldability, thermostability, solubility, and conditional sequence--function alignment. While these Oracles provide efficient and informative proxies for large-scale protein generation, they may not fully capture the biological validity, functional efficacy, or experimental feasibility of the designed proteins. Therefore, future work will focus on wet-lab experiments to empirically validate the structural and biochemical properties of the generated sequences.

\end{document}